\pdfoutput=1

\documentclass[11pt]{article}

\usepackage[final]{acl}

\usepackage{times}
\usepackage{latexsym}

\usepackage[T1]{fontenc}

\usepackage[utf8]{inputenc}

\usepackage{microtype}

\usepackage{inconsolata}

\usepackage{graphicx}

\usepackage{hyperref,amsmath,latexsym,xspace}
\usepackage{mathtools}
\usepackage{amsthm}
\usepackage{tikz}
\usepackage{multirow} 
\usepackage{times}
\usepackage{listings}
\usepackage{subcaption}
\usepackage{xcolor}
\usepackage{adjustbox}
\usepackage{amssymb}
\usepackage{enumerate}
\usepackage{bbm}
\def\bv{{\boldsymbol v}}
\def\bw{{\boldsymbol w}}
\def\bx{{\boldsymbol x}}
\def\by{{\boldsymbol y}}

\makeatletter
\newtheorem*{rep@theorem}{\rep@title}
\newcommand{\newreptheorem}[2]{%
\newenvironment{rep#1}[1]{%
 \def\rep@title{#2 \ref{##1}}%
 \begin{rep@theorem}}%
 {\end{rep@theorem}}}
\makeatother

\newtheorem{theorem}{Theorem}
\newreptheorem{theorem}{Theorem}

\title{Exploring Space Efficiency in a Tree-based Linear Model for 
Extreme Multi-label Classification}

\author{
 \textbf{He-Zhe Lin\textsuperscript{1,2}} \quad
 \textbf{Cheng-Hung Liu\textsuperscript{1}} \quad
 \textbf{Chih-Jen Lin\textsuperscript{1,2}}
\\
\\
 \textsuperscript{1}National Taiwan University\\
 \textsuperscript{2}Mohamed bin Zayed University of Artificial Intelligence
\\
   \texttt{\{b07902028, d07944009, cjlin\}@csie.ntu.edu.tw}
}

\begin{document}
\maketitle
\begin{abstract}
Extreme multi-label classification (XMC) aims to identify relevant subsets from numerous labels.
Among the various approaches for XMC,
tree-based linear models are effective due to their superior
efficiency and simplicity.
However, the space complexity of tree-based methods is not well-studied.
Many past works assume that storing the model 
is not affordable and apply techniques such as pruning to save space, which may lead to performance loss.
In this work, we conduct both theoretical and empirical analyses on the space to store a tree model under the assumption of sparse data, a condition frequently met in text data.
We found that 
some features may be unused 
when training binary classifiers in a tree method,
resulting in zero values in the weight vectors.
Hence, storing only non-zero elements can greatly save space.
Our experimental results indicate that tree models can require less than 10\% of the size of the standard one-vs-rest method
for multi-label text classification.
Our research provides a simple procedure to estimate the size of a tree model before training any classifier in the tree nodes.
Then, if the model size is already acceptable, this approach can help avoid modifying the model through weight pruning or other techniques.
\end{abstract}

\section{Introduction}

Extreme multi-label classification (XMC) focuses on tagging a given instance with a
relevant subset of labels from an extremely large label set.
There is a wide range of applications, from online retail search systems \citep[]{WCC21a}
to automatically tagging labels on a given article or web page \citep[]{HJ16a}.
XMC problems are commonly encountered in real-world applications,
especially in the area of text data.
Notably, \citet{KB16a} provide many text data sets for XMC.

Many methods have been proposed to solve XMC for text data.
For example, neural networks, particularly pre-trained language models, are effective due to their ability to understand context \citep{IC22B}.
However, linear methods with bag-of-words features
remain very useful for XMC due to
their simplicity and
superior efficiency \citep{HFY22a}.
Linear methods are also competitive in certain circumstances \citep{WCC21a}.
This motivates us to investigate the time and space complexity of
linear methods in XMC problems in this work.

Among linear methods for multi-label problems, 
the simplest one-vs-rest (OVR) setting 
treats a multi-label problem with $L$ labels as $L$ 
independent binary problems, 
learning a weight vector $\bw_j$
for each label $j \in \{1, \dots, L\}$.
However, for OVR in XMC,
current computing resources become unaffordable
as the training time and storage for $\bw_j$'s grow linearly with $L$.

An important line of research on reducing time and space is to construct a label tree \citep{YP18a, SK20a, HFY22a}.
The label tree recursively decomposes the XMC problem into smaller ones, and each node in the tree only handles a subset of labels so that the overall training time grows with respect to $O(\log^2 L)$ instead of $O(L)$ as shown in \citet{YP18a}.
However, we explain in Section~\ref{subsec:tree-based-methods} that more classifiers are needed for a label tree compared to OVR.
If weights of these classifiers are dense vectors in the same dimensionality of the input features, storing a tree-based model would need more space then OVR.
To tackle this issue, nearly all past works on tree-based models such as \citet{YP18a, SK20a, HFY22a} apply weight pruning to change small non-zero entries to zero.
However, our experiment in Section~\ref{sec:technique-for-reducing-model-size} shows that weight pruning via an improper threshold can result in varying degrees of performance drops.
This situation motivates us to study the model size of a tree-based model in practice.

In this work, we focus on a less-studied aspect -- the needed space to store a linear tree model for sparse data.
Data sparsity is ubiquitous in XMC.
For example, for the bag-of-words features widely used in text classification, each instance only contains several non-zero entries.
In the tree methods, each node handles a subset of labels and trains on instances corresponding to them.
Since the data are sparse, the training subset for a node may have some unused features, causing a reduction on the feature space.
Thus, if weight vectors are mapped back to the input feature space, many elements are zeros and we can use a sparse format to save the space.
This property suggests that a tree model -- despite having more classifiers -- may still take less space than an OVR model.
Surprisingly, none of the papers investigate the actual size of a tree model; instead, pruning is directly applied.

For sparse training data, 
we study this issue through theoretical analysis and experiments.
Our main finding is that for sparse data,
a tree-based linear model is smaller than what people thought before.
Compared with the OVR approach,
the model size is 10\% or even less for problems with
a large number of labels.
Our research suggests that one should not impose the pruning procedure without checking the model size, which can be easily calculated before the real training.

The outline is as follows.
Section~\ref{sec:linear-methods-for-XMC} defines the XMC problem and introduces two main linear approaches: OVR and tree-based methods.
Section~\ref{sec:complexity-analysis} compares both the time complexity and the model size of the two approaches.
We review some techniques to reduce the model size and discuss their potential issues in Section~\ref{sec:technique-for-reducing-model-size}.
We explain the space-efficiency of tree-based linear models for sparse data in Section~\ref{sec:inherent-pruning}.
Section~\ref{sec:experiment-results} presents experimental results and conclusions are provided in Section~\ref{sec:conclusions}.
This work is an extension of the first author's master thesis \cite{HZL24a}.
Additional materials including programs used for experiments are available at \url{https://www.csie.ntu.edu.tw/~cjlin/papers/multilabel_tree_model_size/}.
\section{Linear Methods for XMC}
\label{sec:linear-methods-for-XMC}
To address XMC, there are two main categories for linear methods: the one-vs-rest (OVR) method and tree-based methods.
In this section, we will introduce how these methods work.
Consider the multi-label classification problem involving $L$ labels and  $\ell$ training instances $\{(\bx_i, \by_i)\}_{i=1}^\ell$, where  $\bx_i \in \mathbb{R}^n$ represents the feature vector with $n$ features, and $\by_i \in \{-1, 1\}^L$ is the $L$-dimensional label vector.
For each label vector $\by_i$, if $\bx_i$ is associated with label $j$, then the $j$-th component of $\by_i$, denoted by $y_{ij}$, is $1$; otherwise, $y_{ij} = -1$.

\subsection{One-vs-rest (OVR) Method}
\label{subsec:one-vs-rest}

For each label $j \in \{1, \dots, L\}$, the linear one-vs-rest approach trains a binary linear classifier by using instances with label $j$ as positive instances and the others as negative ones.
The learned weight vector $\bw_j \in \mathbb{R}^n$ is obtained by solving the minimization problem 
\begin{equation}
    \label{eq:binary-relevance-loss}
    \begin{split}
    \min_{\bw} f(\bw) = \sum_{i=1}^\ell \xi(y_{ij}\bw^T \bx_i) + R_\lambda(\bw),
    \end{split}
\end{equation}
where $\xi(\cdot)$ is the loss function and $R_{\lambda}(\bw)$ denotes the regularization function with parameter $\lambda$.
As we will discuss in Section~\ref{sec:complexity-analysis}, OVR is inefficient since both the time and space complexity grow linearly in the number of labels.

\subsection{Tree-based Methods}
\label{subsec:tree-based-methods}
To reduce the needed time and space, many works \citep{YP18a, SK20a, HFY22a} focused on tree-based methods.
These methods use divide-and-conquer to recursively break down a multi-label problem with $L$ labels into several smaller sub-problems involving subsets of labels.
Figure~\ref{fig:tree} is an example of a constructed tree with nine labels.
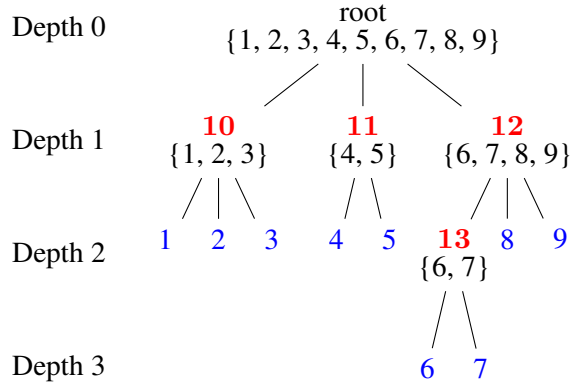
\begin{figure}[t]
    \centering
    \begin{tikzpicture}
        [level distance=1.5cm,
            level 1/.style={sibling distance=1.9cm},
            level 2/.style={sibling distance=0.7cm}]
        \node {\shortstack{root\\ \{1, 2, 3, 4, 5, 6, 7, 8, 9\}}}
        child {node {\shortstack{${\bf {\color{red}10}}$\\ \{1, 2, 3\}}}
            child {node {\shortstack{{\color{blue}1}\\\phantom{\{\}}}}
                child [grow=left, xshift=0.1cm] {node {Depth 2} edge from parent[draw=none]
                    child [grow=down] {node {Depth 3} edge from parent[draw=none]}
                    child [grow=up] {node {Depth 1} edge from parent[draw=none]
                        child [grow=up] {node {Depth 0} edge from parent[draw=none]}
                    }
                }
            }
            child {node {\shortstack{{\color{blue}2}\\\phantom{\{\}}}}}
            child {node {\shortstack{{\color{blue}3}\\\phantom{\{\}}}}}
        }
        child {node {\shortstack{${\bf {\color{red}11}}$ \\ \{4, 5\}}}
            child {node {\shortstack{{\color{blue}4}\\\phantom{\{\}}}}}
            child {node {\shortstack{{\color{blue}5}\\\phantom{\{\}}}}}
        }
        child {node {\shortstack{${\bf {\color{red}12}}$ \\ \{6, 7, 8, 9\}}}
            child {node {\shortstack{${\bf {\color{red}13}}$ \\ \{6, 7\}}}
                child {node {{\color{blue}6}}}
                child {node {{\color{blue}7}}}
            }
            child {node {\shortstack{{\color{blue}8}\\\phantom{\{\}}}}}
            child {node {\shortstack{{\color{blue}9}\\\phantom{\{\}}}}}
        }
        ;
    \end{tikzpicture}
    \caption[A label tree with nine labels.]{
    A label tree with nine labels. We set the number of clusters $K=3$ at each node for the label partition.
    In the figure, each internal node colored red is associated with a label subset, and each leaf node colored blue corresponds to a single label.
    }
    \label{fig:tree}
\end{figure}
A possible procedure of construction is as follows.
\begin{itemize}
    \item For each label $j = 1, \dots, L$, we compute the label representation $\bv_j$ by summing up the feature vectors of the instances associated with label $j$ followed by normalization. That is, 
    \begin{equation}
        \label{eq:label-representation}
        \begin{split}
        \bv'_j = \sum\limits_{i: y_{ij}=1}\bx_i \text{ and }
        \bv_j = \frac{\bv'_j}{\| \bv'_j \|_2}.
        \end{split}
    \end{equation}
    \item Starting from the root node, we apply a clustering method to all $\bv_j$'s to partition all labels into $K$ clusters. 
    Each cluster corresponds to a child node and contains a subset of labels.
    \item Each child node is recursively partitioned into $K$ clusters\footnote{Each child node may have different $K$'s, but for simplicity, we apply the same number of partitions here.} until either of the following termination conditions happen.
        \begin{itemize}
        \item The number of labels in a node is no more than $K$.
        \item The node reaches depth-$(d_\text{max}-1)$, where $d_\text{max}$ is a pre-set maximum tree depth.\footnote{We follow the setting in \citet{SK20a}. In contrast, \citet{YP18a} set a parameter $M$ to stop growing the tree if the number of labels in a node is less than $M$.}
        \end{itemize}
    If a node stops partitioning but still has multiple labels in the subset, we add a child node for each label in the subset.
    Therefore, every leaf node in the tree corresponds to a single label.
\end{itemize}
We denote $d$ as the actual depth of the tree.
It may be smaller than the specified depth $d_\text{max}$ if before $d_\text{max}$ each leaf node already has only one label.

\begin{table}[t]
\begin{adjustbox}{max width=0.49\textwidth,center}
\begin{tabular}{c|c|c}
& \begin{tabular}{c}positive\end{tabular} 
& \begin{tabular}{c}negative\end{tabular}\\
\hline
 $\bf 13$ & \begin{tabular}{c}with labels\\ 6 or 7 \end{tabular} & \begin{tabular}{c}with labels 8 or 9 but\\ neither with label 6 nor 7 \end{tabular} \\
 \hline
 8 & with label 8 & \begin{tabular}{c}with labels 6, 7 or 9\\ but not with label 8 \end{tabular} \\
 \hline
 9 & with label 9 & \begin{tabular}{c}with labels 6, 7 or 8\\ but not with label 9 \end{tabular} \\
\end{tabular}
\end{adjustbox}
\caption[An example for training an OVR model for the node of meta-label $\bf 12$ in Figure \ref{fig:tree}.]{An example for training an OVR model for the node of meta-label $\bf 12$ in Figure \ref{fig:tree},
where the node has 3 children including meta-label $\bf 13$, label 8, and label 9. For example, an instance having labels 6, 7 and 8 is regarded as positive in the first two binary problems since it has labels within the label subsets for meta-labels {\bf 13} and label 8.}
\label{table:example-train-internal-node-of-metalabel-12}
\end{table}

For easy discussion, for any node which is neither the root nor a leaf node, we tag a meta-label on it; see an example in Figure~\ref{fig:tree}.
In the label tree, each internal (non-leaf) node with $r$ child nodes corresponds to a multi-label classification problem of $r$ (meta)-labels.
The $r$ children may include meta-labels (i.e., internal nodes) and the original labels (i.e, leaf nodes).
Similar to the OVR setting, we train a binary problem for each of the $r$ branches.
However, instead of using the whole training set, we only use instances with labels in the node's label subset.
Take the node of meta-label $\bf 12$ in Figure~\ref{fig:tree} as an example.
We only use instances having at least one of labels \{6, 7, 8, 9\} to train three binary problems, which correspond to the three children including meta-label $\bf 12$, label 8 and label 9.
The positive/negative instances of each binary problem are shown in Table~\ref{table:example-train-internal-node-of-metalabel-12}.
As our focus is on the time and space complexity for constructing a tree model, we omit discussing the details of the  prediction procedure.
Readers can check \citet{YP18a, SK20a,HFY22a} for details.

\section{Time and Space Analysis for Linear Methods}
\label{sec:complexity-analysis}
In this section, we compare both the training time and the model size for OVR and tree-based methods.
The training-time analysis clearly demonstrates why people favor tree-based methods over OVR, while the model-size discussion highlights the expensive space cost with tree-based methods.
Although these methods are well documented, our descriptions may be the first to discuss the complexity of the model size in detail.

Through this section, we follow \citet{YP18a} to assume the constructed label tree of depth $d$ is balanced, as shown in Figure~\ref{fig:K-ary-tree}.
That is, given the number of clusters $K$, we assume an ideal situation so that at each node, a clustering method splits its label subset to $K$ equally sized clusters.
By this design, a specified $d_\text{max}$ that is not too large to exhaust all labels results in a tree with depth $d=d_\text{max}$.
Thus, in our analysis of using the tree in Figure~\ref{fig:K-ary-tree}, only $d$ appears in the time and space complexity.

We further assume that each training instance $\bx_i$ has $\bar{n}$ non-zero elements on average.
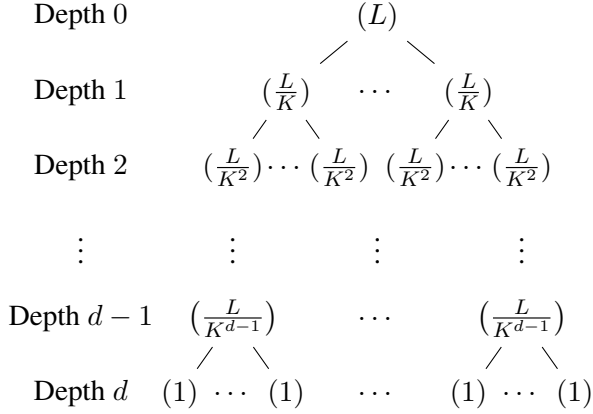
\begin{figure}[t]
    \begin{center}
    \begin{tikzpicture} 
      [level distance=1cm,
        level 1/.style={sibling distance=1.2cm},
        level 2/.style={sibling distance=0.7cm}]
    \node {$(L)$}
    child {node {$(\frac{L}{K})$} {
        child {node {$(\frac{L}{K^2})$} {
          child {node {$\vdots$} edge from parent[draw=none]
              child {node {$\left(\frac{L}{K^{d-1}}\right)$} edge from parent[draw=none]
                child {node {}edge from parent[draw=none]}
                child {node {$(1)$}}
                child {node {$\dots$}edge from parent[draw=none]}
                child {node {$(1)$}}
              child [grow=left] {node {} edge from parent[draw=none]
                  child [grow=left] {node {Depth $d-1$} edge from parent[draw=none]
                  child [grow=down] {node {Depth $d$} edge from parent[draw=none]}
                  child [grow=up] {node {$\vdots$} edge from parent[draw=none]
                    child [grow=up] {node {Depth $2$} edge from parent[draw=none]
                      child [grow=up] {node {Depth $1$} edge from parent[draw=none]
                        child [grow=up] {node {Depth $0$} edge from parent[draw=none]}
                      }
                    }
                  }
                  }
              }
            }
          }
        }
        }
        child {node {$\dots$} edge from parent [draw=none]}
        child {node {$(\frac{L}{K^2})$} {
            child {node {} edge from parent[draw=none]
              child {node {} edge from parent[draw=none]
                child {node {}edge from parent[draw=none]} 
              }
            }
          }
        }
      }
    }
    child {node {$\dots$} edge from parent [draw=none]
      child {node {} edge from parent [draw=none]}
      child {node {} edge from parent [draw=none]
                child {node {$\vdots$} edge from parent [draw=none]
                child {node {} edge from parent [draw=none]}
                child {node {$\dots$} edge from parent [draw=none]
                    child {node {$\dots$} edge from parent [draw=none]}}
                child {node {} edge from parent [draw=none]}
                }
            }
      child {node {} edge from parent [draw=none]}
    }
    child {node {$(\frac{L}{K})$} {
      child {node {$(\frac{L}{K^2})$} {
              child {node {} edge from parent[draw=none]
                child {node {} edge from parent[draw=none]
                  child {node {}edge from parent[draw=none]}
                }
              }
            }
          }
      child {node {$\dots$} edge from parent [draw=none]}
      child {node {$(\frac{L}{K^2})$} {
            child {node {$\vdots$} edge from parent[draw=none]
              child {node {$\left(\frac{L}{K^{d-1}}\right)$} edge from parent[draw=none]
                child {node {$(1)$}}
                child {node {$\dots$}edge from parent[draw=none]}
                child {node {$(1)$}}
              }
            }
          }
        }
      }
    }
    ;
    \end{tikzpicture}
    
    \end{center}
    \caption[Illustration for a depth $d$ balanced label tree with number of clusters $K$.]{\label{fig:K-ary-tree} Illustration for a depth $d$ balanced label tree with number of clusters $K$. The number within $(\cdot)$ at each node means the size of the node's label subset.
    Nodes from depth $0$ to depth $(d-2)$ all have $K$ children.
    Because the tree needs to be terminated at depth $d$, each node at depth $d-1$ has $L/K^{d-1}$ children.}
\end{figure}

Table~\ref{tbl:comparsion} gives a summary for the training time complexity and the model size.
We now explain each entry in detail.
\begin{table}[t]
\begin{center}
\begin{tabular}{c|c|c}
    Model & OVR & Balanced-tree\\
    \hline
    Time & $O(L\ell \bar{n})$ & Eq. \eqref{eq:time-complexity}\\
    \hline 
    \begin{tabular}{c} Model\\size \end{tabular} & $Ln$ & $\left(L+\dfrac{K^d-K}{K-1}\right)n$
\end{tabular}
\caption{\label{tbl:comparsion}A summary of the training time complexity and the model size for linear methods.}
\end{center}
\end{table}

\subsection{Time Analysis}
\label{subsec:time-analysis}
For training an OVR model, we use all instances $\{\bx_i\}_{i=1}^\ell$ to train each of the $L$ binary problems \eqref{eq:binary-relevance-loss}.
According to \citet{CJH08a, LG20a}, the complexity for solving a binary problem is
\begin{align}
  &O(\text{\# nonzero feature values in the training set})\nonumber\\
  \label{eq:binary-problem-time-complexity}
  &\times \text{\# iterations}.
\end{align}
Because the number of iterations is usually not large in practice, we may 
treat it as a constant in the complexity analysis.
Therefore, the time complexity for training an OVR model is
\begin{equation}
    \label{eq:time-for-OVR}
    O(L\ell \bar{n}).
\end{equation}

In contrast, training a tree model is very time-efficient because most binary problems only use part of the training set.
\citet{YP18a} give theoretical time complexity for the training procedure under reasonable assumptions.
However, they do not set the maximum depth in a tree as a termination condition, and their analysis lacks detailed explanation.
In Appendix~\ref{sec:time-analysis} we show that the time complexity for training a tree model of tree depth $d$ is
\begin{equation}
    \label{eq:time-complexity}
    O\left( \ell \bar{n} \log L \times \left( K(d-1) + \dfrac{L}{K^{d-1}} \right) \right).
\end{equation}
Based on \eqref{eq:time-complexity}, the time for training a tree of depth $d = \lceil\log_K L\rceil$ is
\begin{equation}
    \label{eq:lowest-time-complexity}
    O\left(K\ell \bar{n} \log^2 L\right).
\end{equation}

\subsection{Space Analysis}
\label{subsec:space-analysis}
To obtain the model size, we must compute the number of weight vectors in a model and discuss the needed space to store each weight vector.

We explain that in general, the solution of \eqref{eq:binary-relevance-loss} is dense; that is, most elements of $\bw_j$ are non-zeros.
This property is critical for our analysis.
We mentioned in Section~\ref{subsec:time-analysis} that problem \eqref{eq:binary-relevance-loss} is solved by iterative optimization algorithms, where each iteration often involves using the gradient for updating $\bw$.
For easy discussion, let us assume that both $\xi(\cdot)$ and $R_\lambda (\bw)$ are differentiable (e.g., logistic loss and $\ell_2$-regularization).
The gradient $\nabla f(\bw) \in \mathbb{R}^n$ is given by
\begin{equation}
    \label{eq:gradient-of-f}
    \sum_{i=1}^\ell \xi'(y_{ij}{\bw}^T\bx_i) y_{ij}\bx_i + \nabla_{\bw} R_\lambda(\bw).
\end{equation}
If the derivative $\xi'(\cdot)$ is non-zero, which is always the case for logistic loss, as long as a feature occurs in some instances, then the corresponding gradient component is likely non-zero.
The reason is that the sum of several non-zero values usually remains non-zero.
Therefore, regardless of the sparsity of the feature vectors $\{\bx_i\}_{i=1}^\ell$, we roughly have that
\begin{align}
\label{eq:crucial-property}
&\text{if a feature is used in the training set}\\
\Rightarrow \ &\text{corresponding component in $\bw$ is non-zero.}\nonumber
\end{align}
For our discussion, we assume that every feature in the training set (i.e., $\bx_i$, $\forall i$) is used.
This assumption is reasonable because one should remove unused features in the input data.

From the property in \eqref{eq:crucial-property} and our assumption that the training set has no unused features, for storing an OVR model we need to save
\begin{equation}
  \label{eq:space-for-OVR}
  Ln
\end{equation}
weight values for all $L$ weight vectors.

Next, we compute the number of weight vectors $\hat{L}$ in a tree model.
From the discussion in Section~\ref{subsec:tree-based-methods}, each node trains the same number of weight vectors as the node's children, so we have
\begin{align}
\hat{L}
&= \sum_{s\in \text{nodes}} \text{(\# child nodes of $s$)}\nonumber\\
&= \text{(\# nodes in the tree)} - 1\nonumber\\
&= \text{(\# leaf nodes)} + \text{(\# internal nodes)}-1\nonumber\\
&\label{eq:num-weight-vector} = L + \text{(\# meta-labels)}.
\end{align}
From \eqref{eq:num-weight-vector}, we see that training a tree model needs to afford additional storage for the weights of the meta-labels compared to OVR.
For a balanced tree in Figure~\ref{fig:K-ary-tree}, the number of meta-labels is
\begin{equation*}
  \sum_{i=1}^{d-1} K^i = \dfrac{K^d-K}{K-1}.
\end{equation*} 
So the total number of weight values we have to store for a balanced tree model is
\begin{equation}
  \label{eq:space-for-tree}
  \hat{L}n = \left(L+\dfrac{K^d-K}{K-1}\right)n.
\end{equation} 

\section{Techniques for Reducing Model-size and Their Issues}
\label{sec:technique-for-reducing-model-size}

Results in Section~\ref{subsec:space-analysis} indicate that the huge model size is problematic for both OVR and tree-based method in the XMC case.
For example, \citet{HFY22a} mentioned that the well-known Wiki-500k~\citep{KB16a} data set requires approximately 5TB space to store a linear OVR model, which is infeasible for a single computer.
In this section, we briefly review some techniques to reduce the model size and discuss their issues.

For OVR method, \citet{RB17a} use weight pruning to change small values to zero.
By storing only non-zero weights, this strategy effectively reduces the model size.
However, the model may behave differently since
it is fundamentally changed.
Also, this strategy brings other issues such as the selection of suitable pruning thresholds.
Other works \citep{IEY16a, IEHY17a} use L1-regularization to encourage sparse weight vectors without sacrificing performance.
However, according to the results provided in \citet{YP18a}, the slow training and prediction time for XMC problems is still not addressed.

On the other hand, tree-based methods, according to the results in Section~\ref{subsec:space-analysis}, have a larger model compared to OVR.
Since the model size for OVR is already not affordable, past works such as \citet{YP18a, SK20a, HFY22a} may directly assume that reducing the model size is a must.
Therefore, all of them perform weight pruning as in \citet{RB17a}.
However, in the following experiment we show that weight pruning in tree-based methods may cause a performance loss.

\begin{table}[t]
\begin{center}
\tabcolsep = 0.115cm  
\begin{tabular}{@{}c@{}|ccc|ccc@{}}
Pruning & P@1   & P@3   & P@5    & P@1       & P@3       & P@5       \\ \hline
        & \multicolumn{3}{c|}{EUR-Lex}  
        & \multicolumn{3}{c}{AmazonCat-13k} \\ \hline
No      & 82.12 & 68.90 & 57.72 & 92.96     & 79.21     & 64.43     \\
Yes     & 82.08 & 68.83 & 57.50 & 92.95     & 79.19     & 64.40     \\ \hline
        & \multicolumn{3}{c|}{Wiki10-31k}            
        & \multicolumn{3}{c}{Amazon-670k}   \\ \hline
No      & 84.49 & 74.37 & 65.55 & 44.11     & 38.94     & 35.09     \\
Yes     & 84.66 & 74.37 & 65.46 & 43.75     & 38.55     & 34.64    
\end{tabular}
\end{center}

    \caption[Precision scores of the tree model without and with weight pruning.]{Precision scores of the tree model without and with weight pruning. To construct the label tree, we use LibMultiLabel's default parameters $K = 100$ and $d_\text{max} = 6$.
    In general, the partitioning of most nodes stops before reaching $d_\text{max}$
    (For example, Amazon-670k has $L = 670,091$ labels,
    resulting in a balanced tree with a depth of $\log_{100} L \leq 3$.
    So if most nodes reach a depth of 10, the tree would be extremely imbalanced).
    We train each binary problem \eqref{eq:binary-relevance-loss} using squared hinge loss with $\ell_2$-regularization.}
    \label{table:prunning-hurt-performance}
\end{table}
We consider the four smaller data sets used in our experiments; see details in Section~\ref{sec:experiment-results}.
By using the package LibMultiLabel\footnote{\url{https://www.csie.ntu.edu.tw/~cjlin/libmultilabel/}} to conduct the training and prediction, we compare the test precision scores (P@\{1, 3, 5\}) without and with weight pruning in Table~\ref{table:prunning-hurt-performance}.
The threshold for pruning is $0.1$, so any weight within $[-0.1, 0.1]$ is changed to zero.
From Table~\ref{table:prunning-hurt-performance}, the performance drops in all data sets, though the score differences vary across data sets.
In particular, the loss is significant in Amazon-670k.
Note that we choose the $0.1$ threshold by following \citet{YP18a} and \citet{SK20a}.
Our results indicate the difficulty in choosing a suitable threshold.
In fact, the size of the tree model, discussed in Section~\ref{sec:inherent-pruning} and experimentally shown later in Figure~\ref{fig:ratio-fixed-K}, is very small and can be easily stored in one computer.

\section{Inherent Pruning in Tree-based Methods for Sparse Data}
\label{sec:inherent-pruning}
In this section, we explain that for sparse data, the number of weight values needed to be stored in a tree model can be much less than not only the huge value in \eqref{eq:space-for-tree}, but also $nL$, the number of weight values in an OVR model.
We stress that the model size reduction here is not achieved by applying any techniques. 
Instead, it can be regarded as an innate advantage of tree-based methods and we call such reduction ``inherent pruning.''

In the tree-based method, suppose we are training weight vectors at a tree node $u$.
Only a subset of all training instances (specifically, instances having any label in the label subset of $u$) are used.
Because the feature vectors are sparse, some features may have no values in the subset of training instances.
We can remove the unused features before training a multi-label model for the node.
Alternatively, if our optimization algorithm for each binary classification problem \eqref{eq:binary-relevance-loss} satisfies that
\begin{itemize}
  \item the initial $\bw$ is zero, and
  \item for unused features (i.e., feature value are zero across all instances), the corresponding $\bw$ components are never updated, 
\end{itemize}
then we can conveniently feed the subset of data into the optimization algorithm and get a $\bw$ vector with many zero elements.
This way, we keep all weight vectors in all nodes to have the same dimension $n$.
We can collect them as a large sparse matrix for easy use.

We use the name ``inherent pruning'' because the tree-based method itself ``prunes'' weight values for unused features by not updating the corresponding weight components during training.
Our survey shows that few works, except \citet{KJK20a}, mentioned the identity.
Even though, \citet{KJK20a} only briefly said that ``the weight sparsity increases with the depth of a tree ... implies a significant reduction of space'' without further discussion or experiments.


\subsection{Analysis on Balanced Trees}
\label{subsec:case-study}
We theoretically analyze the size of a tree model for sparse data under the following assumptions.

\begin{itemize}
  \item Our analysis follows Figure~\ref{fig:K-ary-tree} to have a $K$-ary balanced tree of depth $d$.
  \item As the tree depth grows, the training subset becomes smaller and the number of used features also reduces.
  Hence, when the number of labels is divided by $K$, we assume that the number of remaining features is multiplied by a ratio $\alpha \in (0, 1)$.
\end{itemize}
We list the information for each depth in Table~\ref{tbl:information-of-K-ary-tree}.

\begin{table}[t]
  \begin{center}
    \begin{tabular}{@{}cccc@{}} 
     Depth & \# nodes &\# children / node & \# features \\
     \hline
     0 & $K^0$ & $K$ & $n$ \\ 
     \hline
     1 & $K^{1}$ & $K$ & $\alpha n$ \\
     \hline
     \vdots & \vdots & \vdots & \vdots\\
     \hline
     $i$ & $K^{i}$ & $K$ & $\alpha^{i} n$\\
     \hline
     \vdots & \vdots & \vdots & \vdots\\
     \hline
     $d-1$ & $K^{d-1}$ & $L/K^{d-1}$ & $\alpha^{d-1} n$
    \end{tabular}
  \end{center}
  \caption{\label{tbl:information-of-K-ary-tree} Depth-wise summary of a tree model with depth $d$.}
\end{table}

Under the assumptions, we discuss the largest possible tree depth, denoted by $D$.
For a tree of depth $D$, because nodes at depth-$D$ cover all labels, each node at depth-$(D-1)$ must contain at least two labels; see the illustration in Figure~\ref{fig:K-ary-tree}.
Therefore, $D$ is the largest possible integer to satisfy
\begin{equation}
  \label{eq:lower-bound-on-L}
  \dfrac{L}{K^{D-1}} \ge 2.
\end{equation}
We then have
\begin{equation}
  \label{eq:def-of-D}
  D = \left\lfloor 1+\log_K \dfrac{L}{2} \right\rfloor.
\end{equation}
On the other hand, the minimum depth of a label tree is $d=2$.
Otherwise, if $d=1$, Table~\ref{tbl:information-of-K-ary-tree} shows that at depth-0, the root node has $L$ children, which is simply the OVR case.
Therefore, the range of the tree depth is $2 \le d\le D$.

We choose the OVR model with $Ln$ weight numbers as the comparison baseline because an OVR model is more space-efficient than a tree model with a dense weight matrix, which takes a space of $\hat{L}n$ weight values in \eqref{eq:space-for-tree}.
Then, we compute the number of non-zero weights in a tree model
\begin{align}
  \phantom{=}&\sum_{i=0}^{d-2} (K^{i})(K)(\alpha^{i} n) + K^{d-1}\left(\dfrac{L}{K^{d-1}}\right) (\alpha^{d-1}n)\nonumber\\
  \label{eq:balanced-tree-nnz}
  &= Kn \cdot \dfrac{(K\alpha)^{d-1}-1}{K\alpha -1} + L\alpha^{d-1}n.
\end{align}
A minor issue is that to have \eqref{eq:balanced-tree-nnz} well defined, we need $K\alpha \neq 1$.
We discuss this exceptional situation in Appendix~\ref{subsec:exceptional-case}.

We compare \eqref{eq:balanced-tree-nnz} with the number of non-zeros in an OVR model by the following ratio:
\begin{equation}
  \label{eq:balanced-nnz-ratio}
  \dfrac{\text{\eqref{eq:balanced-tree-nnz}}}{Ln}
  = \dfrac{K((K\alpha)^{d-1}-1)}{L(K\alpha -1)}+\alpha^{d-1}.
\end{equation}
The following theorem illustrates that the tree model generally contains less non-zero weight values than the OVR model. We give the proof in Appendix~\ref{sec:proof-of-theorems}.
\begin{theorem}
  \label{thm:ratio-smaller-than-1}
  Consider $2 < d \le D$ and assume $K \ge 4$.
  Let $\alpha^*$ be the unique solution in $(0, 1)$ of the equation
  \begin{equation}
    \alpha^{d-2}(K^{d-D} + \alpha) - 1 = 0.
  \end{equation}
  If $\alpha < \max\{2/K, \alpha^*\}$, then the ratio \eqref{eq:balanced-nnz-ratio} is smaller than one.
\end{theorem}
In Theorem~\ref{thm:ratio-smaller-than-1}, we consider $d > 2$ because for $d=2$, the obtained bound on $\alpha$ is in a slightly different form; see details in Appendix~\ref{sec:proof-of-theorems}.
\par Although Theorem~\ref{thm:ratio-smaller-than-1} imposes an upper bound on $\alpha$, we show that the bound is in general close to one.
For example, if $L=2\cdot 10^8$ and $K=100$, we have $D=5$ according to \eqref{eq:def-of-D}.
Then we only need $\alpha < 0.999$ for a depth 2 or 3 tree, $\alpha < 0.996$ for a depth 4 tree and $\alpha < 0.819$ for a depth 5 tree.
Therefore, even if the number of used features is only minorly reduced after each label division, we can significantly lower the size of a tree model from \eqref{eq:space-for-tree} to be smaller than that of OVR.

The following theorem shows that a deeper tree leads to a smaller model, with its proof given in Appendix~\ref{sec:proof-of-theorems}.
\begin{theorem}
  \label{thm:ratio-decreasing-in-d}
  If $\alpha < 1 - 1/(2K)$, the ratio \eqref{eq:balanced-nnz-ratio} is decreasing in $d$ for $2 \le d \le D-2$.
  Specifically, for a tree with depth $d$ within this range, the ratio is smaller than that of a tree of depth $d+1$.
\end{theorem}

\begin{figure}[t]
  \centering
  \includegraphics[width=0.4\textwidth]{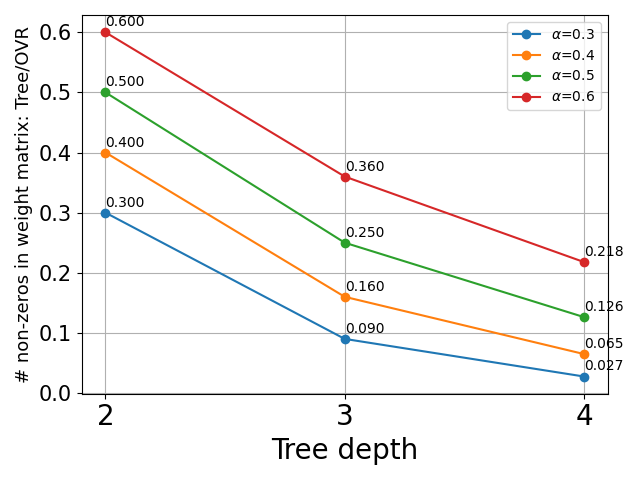}
  \caption[The ratio of number of non-zeros between a tree model and an OVR model.]{\label{fig:fixed-K-with-different-alpha}The ratio of number of non-zeros between a tree model and an OVR model, calculated based on \eqref{eq:balanced-nnz-ratio}. We show the cases for $\alpha=\{0.3, 0.4, 0.5, 0.6\}$.}
\end{figure}

As an illustration of Theorem~\ref{thm:ratio-decreasing-in-d}, we plot the ratio \eqref{eq:balanced-nnz-ratio} with $L = 2\cdot 10^8$, $K=100$ and different $\alpha$'s in Figure~\ref{fig:fixed-K-with-different-alpha}.
For this illustration, $D=5$, so Theorem~\ref{thm:ratio-decreasing-in-d} is applicable for $2\le d \le 3$.
In Figure~\ref{fig:fixed-K-with-different-alpha}, we see that the ratio reduces as the tree grows from $d=2$ to $4$, regardless of the value of $\alpha$.\footnote{See more discussion in Appendix~\ref{sec:comments-on-decreasing-theorem}.}
Later in Section~\ref{subsec:results-and-discussion}, we shall see that the experiments on real-world data align with our theoretical analysis on balanced trees.
\section{Experimental Results}
\label{sec:experiment-results}
\begin{table}[t]
    \centering
    \vspace*{-\baselineskip}
    \begin{tabular}{@{}l@{\hskip 5pt}|@{\hskip 2pt}r@{\hskip 2pt}|@{\hskip 2pt}r@{\hskip 2pt}|@{\hskip 2pt}r@{}}
      \multirow{2}{*}{Data set} & \#training    & \#features & \#labels \\
       & data\quad\  $l$ & $n$ & $L$ \\
      \hline
      EUR-Lex       & 15,449    & 186,104   & 3,956     \\
      AmazonCat-13k & 1,186,239 & 203,882   & 13,330    \\
      Wiki10-31k    & 14,146    & 104,374   & 30,938    \\
      Wiki-500k     & 1,779,881 & 2,381,304 & 501,070   \\
      Amazon-670k   & 490,449   & 135,909   & 670,091   \\
      Amazon-3m     & 1,717,899 & 337,067   & 2,812,281 
    \end{tabular}
  \caption[The statistics of extreme multi-label data sets.]{\label{tbl:dataset-stat} The statistics of extreme multi-label data sets, ordered by the number of labels. Wiki-500k and Amazon-3m are downloaded from the GitHub repository provided in \citet{RY19a}; others are from ``LIBSVM Data: Multi-label Classification.''\footnotemark{} More details are in Appendix~\ref{sec:exp-settings}.}
\end{table}
\footnotetext{\url{https://www.csie.ntu.edu.tw/\~cjlin/libsvmtools/datasets/multilabel.html}}

In this section, we compare the model size of OVR and tree-based methods across several extreme multi-label text data sets.
The statistics for these data sets are listed in Table~\ref{tbl:dataset-stat}.

\begin{figure*}[t]
  \centering
  \begin{subfigure}[h]{0.3\textwidth}
      \centering
      \includegraphics[width=\textwidth]{
        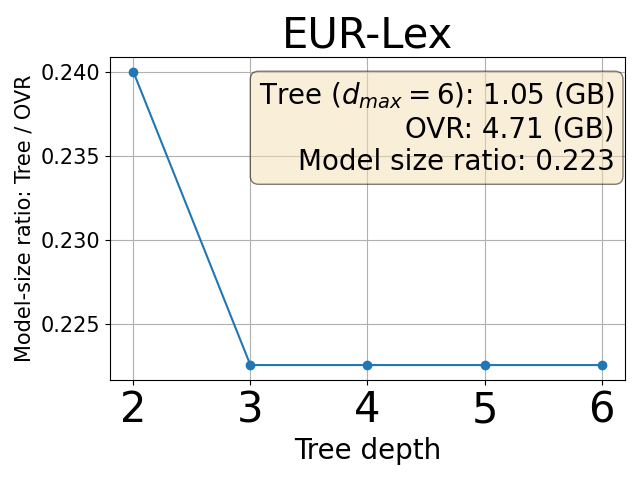}
  \end{subfigure}
  \begin{subfigure}[h]{0.3\textwidth}
      \centering
      \includegraphics[width=\textwidth]{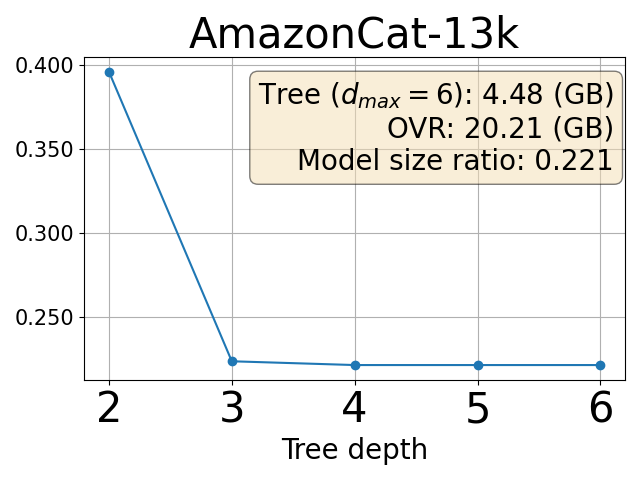}
  \end{subfigure}
  \begin{subfigure}[h]{0.3\textwidth}
      \centering
      \includegraphics[width=\textwidth]{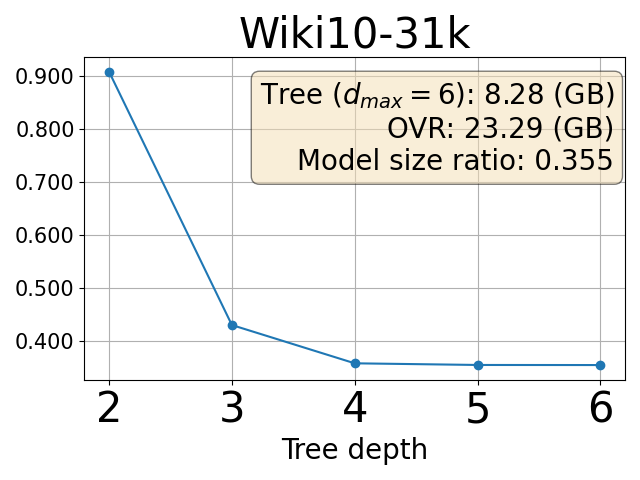}
  \end{subfigure}
  \begin{subfigure}[h]{0.3\textwidth}
      \centering
      \includegraphics[width=\textwidth]{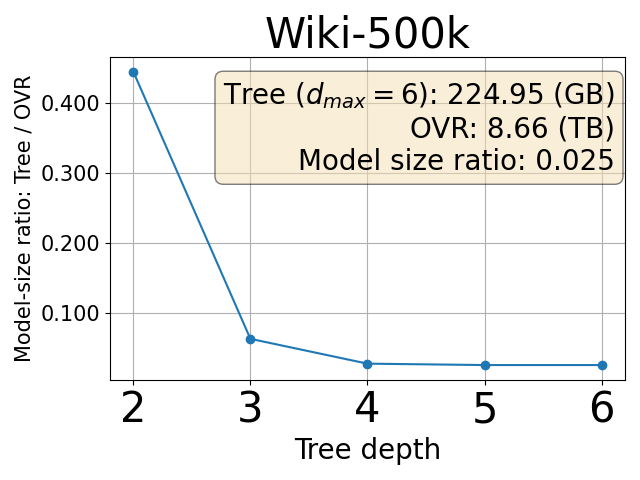}
  \end{subfigure}
  \begin{subfigure}[h]{0.3\textwidth}
    \centering
    \includegraphics[width=\textwidth]{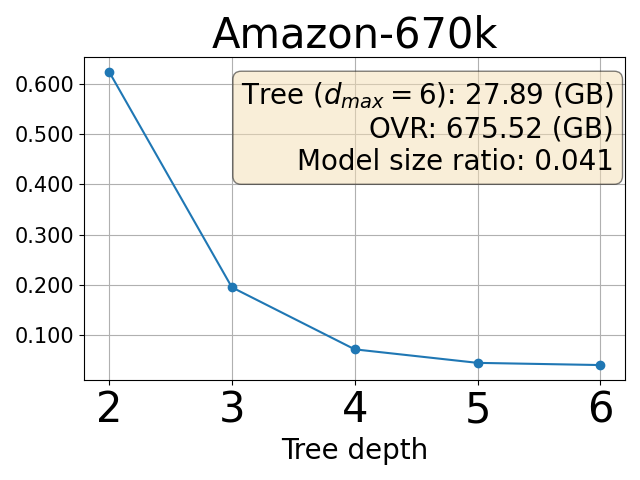}
  \end{subfigure}
    \begin{subfigure}[h]{0.3\textwidth}
    \centering
    \includegraphics[width=\textwidth]{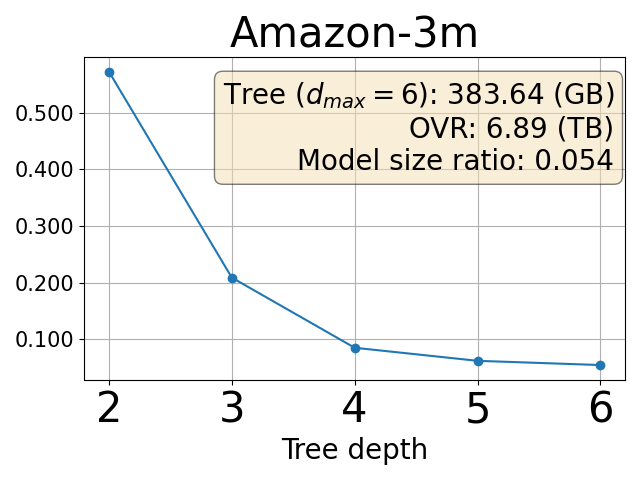}
  \end{subfigure}
  \caption[The ratio between the tree model size and the OVR model size under the {\sf fixed-K} setting.]{\label{fig:ratio-fixed-K} The ratio between the tree model size and the OVR model size with $K=100$ and various $d_\text{max}$. The box in each sub-figure shows the actual model size of tree/OVR models under $d_\text{max} = 6$ and the ratio between the two.}
\end{figure*}

\begin{figure*}[t]
  \centering
  \includegraphics[width=\textwidth]{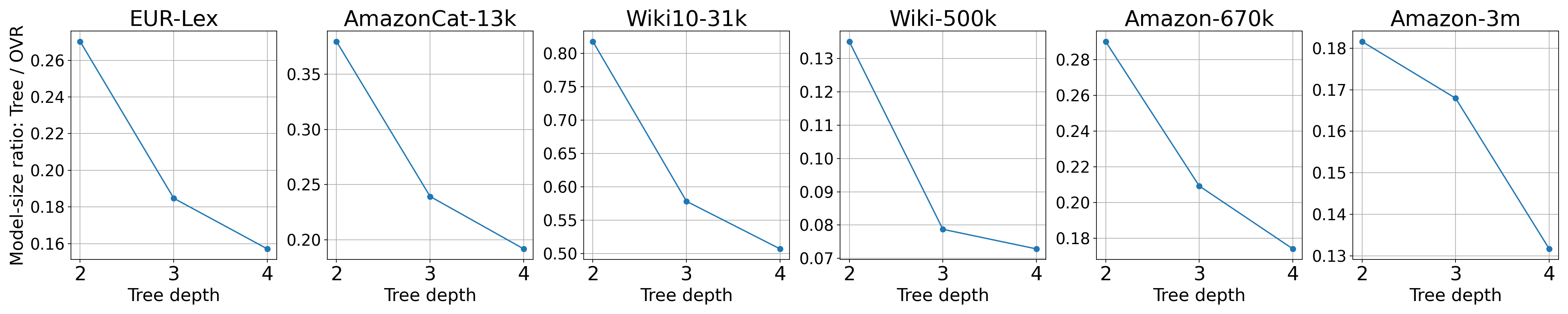}
  \caption[The ratio between the tree model size and the OVR model size under the {\sf varied-K} setting.]{\label{fig:ratio-varied-K} The ratio between the tree model size and the OVR model size with $d_\text{max} = \{2,3,4\}$ and $K=\lceil L^{1/d_\text{max}}\rceil$.}
\end{figure*}

\subsection{Experimental Settings}
For the tree model to be compared, the constructed tree should have high prediction performances; otherwise, claiming that a low-performing tree saves space would be meaningless.
To find out settings with high performances, we must conduct a hyper-parameter search on the number of clusters $K$ and the maximum tree depth $d_{\text{max}}$, etc.
Since developing an effective search procedure is out of the scope of this work, we instead consider tree structures that have been investigated in \citet{SK20a} due to their high performances among almost all data sets.
Specifically, we calculate the model size for the following cases:
\begin{itemize}
  \item {\sf fixed-K}: 
  In \citet{SK20a} they suggested wide trees ($K \ge 100$).
  Therefore, we fix $K=100$ for all data sets and check the model size with $d_\text{max} \in \{2,3,4,5,6\}$.
  This setting can be regarded as an empirical validation for the study in Section~\ref{subsec:case-study}.
  \item {\sf varied-K}:
  In contrast to a fixed $K$, we vary it according to the specified $d_\text{max}$.
  Specifically, we set $K = \lceil L^{1/d_{\text{max}}}\rceil$ by considering $d_\text{max} \in \{2, 3, 4\}$. 
  We do not consider larger $d_\text{max}$ because under this setting of choosing $K$, the performance \citep{SK20a} of deeper trees is poor.
\end{itemize}
We use the software LibMultiLabel\footnote{\url{https://www.csie.ntu.edu.tw/~cjlin/libmultilabel/}} to conduct the experiment.
Currently LibMultiLabel uses K-means algorithm \citep{CE03a} implemented in the package scikit-learn\footnote{\url{https://scikit-learn.org}} for partitioning.
Since the K-means algorithm involves a random selection of $K$ centroids, we conduct the experiments five times on all data sets using different seeds.
More details of our experimental settings are in Appendix~\ref{sec:exp-settings}.

\subsection{Estimating Model Size Prior to Training}
\label{subsec:estimating-model-size}

We explain that the size of a tree-based model can be estimated before training any binary classifiers.
Based on \eqref{eq:crucial-property}, a tight upper bound on the true model size is by summing up each binary problem's used features.
Suppose the weights are stored as double-precision floating-point numbers.
For an OVR model, we have
\begin{equation}
  \label{eq:OVR-nnz}
  nL \times 8 \text{ bytes}
\end{equation}
as the estimation of the model size.
For the tree model, the number of non-zero weights is bounded by
\begin{equation}
  \label{eq:tree-nnz}
  \sum_{u \in \text{nodes}} \text{\# children of $u$}\cdot \text{ \# used features of $u$}.
\end{equation}
However, we also need to store the index of each used feature. 
If we assume a four-byte integer storage for the index, then the model size for a tree model is roughly
\begin{equation}
  \text{\eqref{eq:tree-nnz}} \times 12 \text{ bytes},
\end{equation}
and the ratio of a tree-based model to an OVR model is
\begin{equation}
  \label{eq:model-size-ratio}
  \dfrac{\eqref{eq:tree-nnz}\times 1.5}{nL}.
\end{equation}

\subsection{Empirical Analysis on the Model Size}
\label{subsec:results-and-discussion}
Figure~\ref{fig:ratio-fixed-K} shows the relative size of a tree model compared to OVR under the {\sf fixed-K} setting.
Besides, in a separate box of each sub-figure, we give the actual model size of an OVR model and a tree model with $d_\text{max} = 6$.
We find that the memory consumption is indeed acceptable for a single computer.
For smaller data sets, the tree model size is around 20 to 40\% of the OVR model, while for large data sets, the ratio is lower than 10\%.
Our results fully support the space efficiency of tree-based methods on sparse data.
Moreover, as $d_\text{max}$ grows in the early stage, the ratio \eqref{eq:model-size-ratio} significantly drops.
This observation is consistent with our analysis in Section~\ref{subsec:case-study}.
However, if we further grow $d_\text{max}$, the model size may not change much.
The reason is that the partitioning finishes before reaching the maximum depth, causing the actual tree depth $d$ being smaller than $d_\text{max}$.

For the {\sf varied-K} setting, the relative size between the two models is presented in Figure~\ref{fig:ratio-varied-K}. 
The model size still keeps decreasing as the tree becomes deeper.
A comparison between Figure~\ref{fig:ratio-fixed-K} and Figure~\ref{fig:ratio-varied-K} shows that, under the same $d_\text{max}$, in general a larger $K$ leads to a smaller model.
For example, if $d_\text{max}=4$, the setting of $K = \lceil L^{1/d_{\text{max}}}\rceil$ of {\sf varied-K} leads to $K\le 100$ (as our largest $L$ is less than $10^8$); i.e., smaller than {\sf fixed-K}.
The ratios in Figure~\ref{fig:ratio-fixed-K}, especially for the larger four sets, are clearly smaller than those in Figure~\ref{fig:ratio-varied-K}.
Although a larger $K$ brings more binary problems to train in a tree model, it also leads to more unused features after a label division.
Apparently, within the scope of our experimental settings, the increase of unused features has a higher influence on the model size than the more binary problems.

\subsection{Empirical Study on the Reduction Rate $\alpha$}
In Section~\ref{subsec:case-study}, for balanced trees, we assume that the number of used features is multiplied by $\alpha \in (0, 1)$ when the number of labels is divided by $K$.
In Appendix~\ref{sec:empirical-alpha}, we empirically study the reduction of used features for unbalanced label trees on real data.
\section{Conclusions}
\label{sec:conclusions}
In this work, we identify that the many unused features are the main reason for the space-efficiency of tree-based methods under sparse data conditions.
Our findings indicate that, for large data sets, the size of a tree model can be reduced to just 10\% of the size of an OVR model.
In practice, one can first calculate the tree model size as soon as the label tree is constructed.
By doing so, we can check whether the model size exceeds the available memory before training any binary problem.
This approach avoids directly changing the trained weights such as pruning, which carries the risks of compromising the performance.

\clearpage
\section*{Limitations}
Though we can estimate the model size for a tree-based model before training the model, it may be still time-consuming to generate the constructed label tree because K-means algorithm takes considerable time.
Besides, analyzing the tree model size using a non-constant feature reduction rate (depending on the number of clusters $K$ and the depth $d$) could be a possible direction.
\section*{Acknowledgments}
This work was supported in part by National Science and Technology Council of Taiwan grant 110-2221-E-002-115-MY3.

\bibliography{sdp/sdp}

\pagebreak
\clearpage
\appendix
\section{Time Analysis on Tree Models}
\label{sec:time-analysis}
We explain the time complexity \eqref{eq:time-complexity} for constructing a balanced-tree model with tree depth $d$.
As shown in Figure~\ref{fig:K-ary-tree}, each node from depth-$0$ to depth-$(d-2)$
contains $K$ children, and each node at depth-$(d-1)$ has at most $\lceil L/K^{d-1}\rceil$ children.
Our assumptions on the training data are based on \citet{YP18a}.
We assume that
\begin{itemize}
  \item each training instance $\bx_i$ has $\bar{n}$ non-zero elements on average, and
  \item the average number of relevant labels for each instance is bounded by
  $c\log L$, where $c$ is a constant.
\end{itemize}

As in Section~\ref{subsec:tree-based-methods}, there are three parts to construct a tree model.
\begin{enumerate}[1.]
    \item Computing the label representations as in \eqref{eq:label-representation} costs $O(l \bar{n} \log L)$-time.
    \item To build a label tree,
    K-means clustering \citep{JM67a} is used to
    recursively partition the labels.
    The K-means algorithm has several iterations.
    For each iteration, one need to calculate the distance from all label representations
    to the center of each cluster.
    So learning K-means clustering from depth-$1$ to depth-$(d-1)$ costs
    \begin{equation*}
        O(\text{nnz}(V)\times K \times \# \text{iterations}\times d)\text{-time},
    \end{equation*}
    where $V$ is the label representation matrix.
    \item The last part is to train classifiers for the tree nodes.
    According to \citet{CJH08a, LG20a}, the complexity for solving a binary problem is \eqref{eq:binary-problem-time-complexity}.
    Because the number of iterations is usually not large in practice, we may 
    treat it as a constant in the complexity analysis.
    Then, we compute the training complexity for each depth-$0$ to depth-$(d-1)$.
    \begin{itemize}
      \item Depth-$0$: We have to train $K$ classifiers.
      For each binary problem we use all $\ell$ training instances.
      Therefore, solving $K$ binary problems costs
      \begin{equation}
        \label{eq:time-for-training-root}
        O(K\ell \bar{n}).
      \end{equation}
      \item Depth-$1$ to depth-$(d-2)$: At depth-$i$ there are $K^{i}$ nodes.
      The corresponding $K^i$ label subsets of the nodes form a partition of all $L$ labels.
      Since we assume that each instance $\bx_i$ has less than $c\log L$ labels on average, $\bx_i$ is used in no more than $c\log L$ nodes.
      Therefore, by summing the number of used training instances in the $K^{i}$ nodes, we get a total of $\ell c\log L$.
      Finally, for each node we have $K$ binary problems to train, so the time complexity for training all $K^{i}$ nodes ($K^{i+1}$ binary problems in total) is
      \begin{equation}
        \label{eq:time-for-training-internal-nodes}
        O(K\ell c(\log L) \bar{n}).
      \end{equation}
      \item Depth-$(d-1)$: This is similar to the previous case.
      The only difference is that for each node we have $\lceil L/K^{d-1}\rceil$ problems to solve (instead of $K$), which leads to a complexity of
      \begin{equation}
        \label{eq:time-for-training-leafs}
        O\left(\dfrac{L}{K^{d-1}}\ell c(\log L) \bar{n}\right).
      \end{equation}
    \end{itemize}

    Therefore, the time complexity for training is
    \begin{equation*}
      \text{\eqref{eq:time-for-training-root}} + (d-2)\cdot \text{\eqref{eq:time-for-training-internal-nodes}} + \text{\eqref{eq:time-for-training-leafs}} = \text{\eqref{eq:time-complexity}}.
    \end{equation*}
    By the inequality of arithmetic and geometric means, the inner term in \eqref{eq:time-complexity} can be written as
    \begin{equation}
        \label{eq:AM-GM-inequality}
        \underbrace{K+\cdots + K}_\text{$(d-1)$ times} +\frac{L}{K^{d-1}}
        \geq d \sqrt[d]{L}.
    \end{equation}
    The equality in \eqref{eq:AM-GM-inequality} holds when
    \begin{equation*}
        K = \frac{L}{K^{d-1}},
    \end{equation*}
    which means $d = \log_K L$.
    Under this value of $d$, the complexity is as in \eqref{eq:lowest-time-complexity}.
\end{enumerate}


\clearpage
\section{Proof of Theorems}
\label{sec:proof-of-theorems}

\subsection{Proof of Theorem~\ref{thm:ratio-smaller-than-1}}
\label{subsec:proof-of-thm-1}
Here we give a full version of Theorem~\ref{thm:ratio-smaller-than-1} by including the case of $d=2$.
\begin{reptheorem}{thm:ratio-smaller-than-1}
  \label{thm:ratio-smaller-than-1-full-version}
  The ratio \eqref{eq:balanced-nnz-ratio} is smaller than one if any of the following conditions holds.
  \begin{enumerate}[(i)]
    \item \label{cond-1-full-version} $d=2$ and 
    \begin{equation}
      \label{eq:alpha-bound-for-depth-2}
      \alpha < 1 - 1/(2K^{D-2}).
    \end{equation}
    \item \label{cond-2-full-version} $d>2$, $K \ge 4$ and $\alpha < \max\{2/K, \alpha^*\}$, where $\alpha^*$ is the unique solution in $(0, 1)$ of the equation
    \begin{equation}
      \alpha^{d-2}(K^{d-D} + \alpha) - 1 = 0.
    \end{equation}
  \end{enumerate}
\end{reptheorem}  

\begin{proof}[Proof of Theorem~\ref{thm:ratio-smaller-than-1}]
If $d=2$, we have
\begin{align}
  \label{eq:proof-for-depth-2}
  \text{\eqref{eq:balanced-nnz-ratio}}
  = \dfrac{K}{L} + \alpha \le \dfrac{K}{2K^{D-1}} + \alpha
  < 1,
\end{align}
where the inequalities follow from \eqref{eq:lower-bound-on-L} and \eqref{eq:alpha-bound-for-depth-2}.

If $d> 2$, from the condition 
\begin{equation}
  \label{eq:condition-on-alpha}
  \alpha < \max\{2/K, \alpha^*\},
\end{equation}
we consider two cases.
\begin{enumerate}[a.]
  \item $\alpha < 2/K$: 
  In this case, $\alpha$ of course satisfies \eqref{eq:condition-on-alpha}.
  We see that the number of non-zeros in a tree model \eqref{eq:balanced-tree-nnz} is increasing in $\alpha$. 
  Therefore, by $\alpha < 2/K$ we have
    \begin{align}
      \text{\eqref{eq:balanced-nnz-ratio}}
      &< \dfrac{K (K(2/K))^{d-1}}{L(K(2/K)-1)} + \left(\dfrac{2}{K}\right)^{d-1}\nonumber\\
      &\le \label{eq:tmp1-in-small-alpha} \dfrac{2^{d-2}}{K^{D-2}} + \left(\dfrac{2}{K}\right)^{d-1}\\
      &= \left(\dfrac{2}{K}\right)^{d-2}\left( \dfrac{1}{K^{D-d}} + \dfrac{2}{K} \right)\nonumber\\
      &\le \label{eq:tmp2-in-small-alpha} \left(\dfrac{2}{4}\right)^{3-2}\left( 1 + \dfrac{2}{4} \right)
      < 1,
    \end{align}
    where \eqref{eq:tmp1-in-small-alpha} is from \eqref{eq:lower-bound-on-L} and \eqref{eq:tmp2-in-small-alpha} follows from $K \ge 4$ and $2 < d \le D$.
  \item We consider $\alpha$ satisfying \eqref{eq:condition-on-alpha} but not in the previous case.
  We must have 
  \begin{equation*}
    2/K \le \alpha < \alpha^*.
  \end{equation*}
  The condition on $\alpha$ implies that $K\alpha \ge 2$ and thus
  \begin{equation}
    \label{eq:tmp1-in-big-alpha}
    \begin{split}
    \dfrac{K\alpha}{K\alpha-1} &= 1+\dfrac{1}{K\alpha-1} \\
    &\le 1 + \dfrac{1}{2-1} = 2.
    \end{split}
  \end{equation}
  Then, 
    \begin{align}
      \text{\eqref{eq:balanced-nnz-ratio}}
      &\le \label{eq:tmp2-in-big-alpha} \dfrac{K(K\alpha)^{d-1}}{2K^{D-1}(K\alpha-1)} + \alpha^{d-1}\\
      &= \dfrac{K\alpha}{K\alpha-1}\dfrac{K(K\alpha)^{d-2}}{2K^{D-1}} + \alpha^{d-1}\nonumber\\
      &\le \label{eq:tmp3-in-big-alpha} 2\cdot \dfrac{(K\alpha)^{d-2}}{2K^{D-2}} + \alpha^{d-1}\\
      &= \label{eq:upper-bound-case-2} \alpha^{d-2}(K^{d-D} + \alpha),
    \end{align}
  where \eqref{eq:tmp2-in-big-alpha} follows from \eqref{eq:lower-bound-on-L} and \eqref{eq:tmp3-in-big-alpha} is from \eqref{eq:tmp1-in-big-alpha}.
  Consider the function
  \begin{equation*}
    f(\alpha) = \alpha^{d-2}(K^{d-D} + \alpha) - 1.
  \end{equation*}
  Clearly, $f(\alpha)$ is strictly increasing in $[0, 1]$, $f(0) = -1$, and $f(1) > 0$.
  Therefore, $f$ has a unique root $\alpha^*$ in $(0, 1)$ and we have 
  \begin{equation*}
    \text{\eqref{eq:upper-bound-case-2}} < 1 \text{ if } \alpha < \alpha^*.
  \end{equation*}
\end{enumerate}
\end{proof}

\subsection{Theorem~\ref{thm:ratio-smaller-than-1} for the Exceptional Case of $K\alpha=1$}
\label{subsec:exceptional-case}
First we show that the full version of Theorem~\ref{thm:ratio-smaller-than-1} is still valid by slightly changing the proof.
If $d=2$, the ratio \eqref{eq:balanced-nnz-ratio} is in fact
\begin{equation*}
  Kn + L\alpha n,
\end{equation*}
so \eqref{eq:proof-for-depth-2} remains the same.

If $d > 2$, we only need to check the case of $\alpha < 2/K$ because $\alpha = 1/K$ falls into it.
Now \eqref{eq:balanced-tree-nnz} becomes 
\begin{equation}
  \label{eq:balanced-tree-nnz-for-kalpha-1}
  (d-1)Kn + L\alpha^{d-1}n.
\end{equation}
Therefore, the ratio \eqref{eq:balanced-nnz-ratio} becomes
\begin{align}
  \dfrac{\text{\eqref{eq:balanced-tree-nnz-for-kalpha-1}}}{Ln} 
  &= \dfrac{K(d-1)}{L} + \alpha^{d-1}\nonumber \\
  &\le \label{eq:tmp1-in-edge-case} \dfrac{d-1}{2K^{D-2}} + \dfrac{1}{K^{d-1}}\\
  &\le \label{eq:tmp2-in-edge-case} \dfrac{d-1}{2K^{d-2}} + \dfrac{1}{K^{d-1}},
\end{align}
where \eqref{eq:tmp1-in-edge-case} follows from \eqref{eq:lower-bound-on-L}. 
For the first term in \eqref{eq:tmp2-in-edge-case}, we see that for $K \ge 4$ and $d > 2$, the derivative of
\begin{equation*}
  \dfrac{d-1}{2K^{d-2}}
\end{equation*}
with respect to $d$ is
\begin{align*}
  &\phantom{=}\dfrac{2K^{d-2}(1-(\ln K)(d-1))}{4K^{2d-4}}\\
  &< \dfrac{2K^{d-2}(1-\ln K)}{4K^{2d-4}} < 0,
\end{align*}
showing that the term is decreasing in $d$.
Because \eqref{eq:tmp2-in-edge-case} is also decreasing in $K$, we can get an upper bound by considering $K=4$ and $d=3$ to have 
\begin{align}
  \text{\eqref{eq:tmp2-in-edge-case}} 
  &\le \dfrac{3-1}{2\cdot 4^{3-2}} + \dfrac{1}{4^{3-1}} = \dfrac{5}{16} < 1.
\end{align}

Interestingly, though in Section~\ref{subsec:proof-of-thm-1} we do not require any condition on the value $K\alpha$, from the property of our balanced trees, we can prove $K\alpha \ge 1$ in the following theorem. 
Thus, the situation of $K\alpha=1$ is in fact an extreme case.
\begin{theorem}
  \label{thm:lower-bound-of-alpha}
  The feature reduction ratio $\alpha \ge 1/K$ (i.e., $K\alpha \ge 1$).
\end{theorem}
\begin{proof}[Proof of Theorem~\ref{thm:lower-bound-of-alpha}]
Suppose $n$ features are used at a node $u$.
According to the definition of $\alpha$, there are $\alpha n$ features used in each child node of $u$.
Assume for contradiction that $\alpha < 1/K$. 
The total number of used features among all $u$'s child nodes would be less than
\begin{equation*}
  K\times (\alpha n) < n,
\end{equation*}
which leads to a contradiction.  
\end{proof}

\subsection{Proof of Theorem~\ref{thm:ratio-decreasing-in-d}}
We prove the theorem by showing that for 
\begin{equation}
  \label{eq:range-of-decreasing-d}
  d = 2, \dots, D-2,
\end{equation}
the number of non-zero weight values in a tree of depth $d$ is more than a tree of depth $d+1$.
For a tree of depth $d$, the number of non-zeros is 
\begin{equation}
  \label{eq:nnz-for-depth-d}
  \sum_{i=0}^{d-2} (K^i)(K)(\alpha^i n) + L\alpha^{d-1} n,
\end{equation}
and the number of non-zeros in a depth $(d+1)$ tree is
\begin{equation}
  \label{eq:nnz-for-depth-d-add-1}
  \sum_{i=0}^{d-1} (K^i)(K)(\alpha^i n) + L\alpha^{d} n.
\end{equation}
Then we have
\begin{align}
  \text{\eqref{eq:nnz-for-depth-d}}-\text{\eqref{eq:nnz-for-depth-d-add-1}}
  &= L\alpha^{d-1}n - K^{d}\alpha^{d-1}n - L\alpha^d n\nonumber\\
  &= \label{eq:tmp-in-thm-2} \alpha^{d-1}n(L(1-\alpha)-K^d).
\end{align}
Since we have $L \ge 2K^{D-1}$ from \eqref{eq:lower-bound-on-L} and assume $\alpha < 1-1/(2K)$,
\begin{align}
  L(1-\alpha) &\ge 2K^{D-1}(1-\alpha)\nonumber\\
              &> \dfrac{2K^{D-1}}{2K}\nonumber\\
              &= \label{eq:inner-term-in-the-model-size-difference} K^{D-2} \ge K^d,
\end{align}
where the last inequality is from the range of $d$ considered in \eqref{eq:range-of-decreasing-d}.
We then use \eqref{eq:inner-term-in-the-model-size-difference} in \eqref{eq:tmp-in-thm-2} to show that $\text{\eqref{eq:nnz-for-depth-d}}-\text{\eqref{eq:nnz-for-depth-d-add-1}} > 0$.


\section{Comments on Theorem~\ref{thm:ratio-decreasing-in-d}}
\label{sec:comments-on-decreasing-theorem}
\begin{figure}[t]
    \centering
    \includegraphics[width=0.4\textwidth]{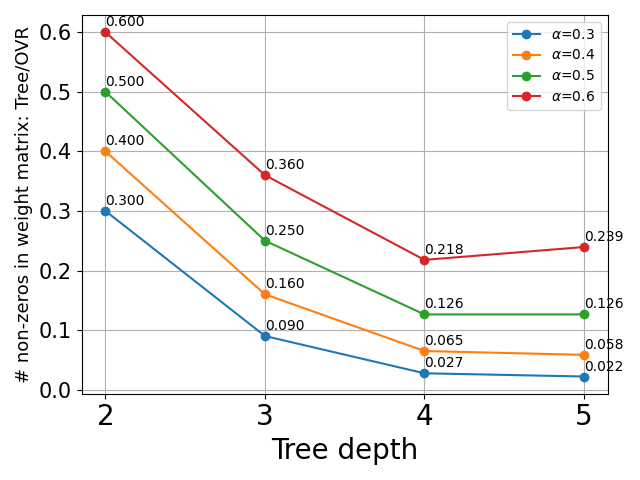}
    \caption[The ratio of number of non-zeros between a tree model and an OVR model.]{\label{fig:fixed-K-with-different-alpha-paper}The ratio of number of non-zeros between a tree model and an OVR model from $d=2$ to $d=D$.}
\end{figure}  
In Theorem~\ref{thm:ratio-decreasing-in-d}, we show the model size decreases for $2 \le d \le D-2$.
However, the ratio for $d=D$ may be larger than the ratio for $d=D-1$.
Figure~\ref{fig:fixed-K-with-different-alpha-paper} is the plot from $d=2$ to $d=D$ using the same example as in Section~\ref{subsec:case-study}.
We see that for $\alpha = 0.6$, the ratio for $d=5$ is slightly larger than $d=4$.
Therefore, our theorem gives the widest interval for a decreasing ratio under the given assumption $\alpha < 1-1/2K$.

\section{Details of Experimental Settings}
\label{sec:exp-settings}

We use LibMultiLabel\footnote{\url{https://www.csie.ntu.edu.tw/~cjlin/libmultilabel/}} version 0.6.0.
For data sets, the specific link of each set is as follows
\begin{itemize}
    \item EUR-Lex:
    \begin{itemize}
        \item Training: \url{https://www.csie.ntu.edu.tw/~cjlin/libsvmtools/datasets/multilabel/eurlex_tfidf_train.svm.bz2}
        \item Testing: \url{https://www.csie.ntu.edu.tw/~cjlin/libsvmtools/datasets/multilabel/eurlex_tfidf_test.svm.bz2}
    \end{itemize}
    \item Wiki10-31k:
    \begin{itemize}
        \item Training: \url{Training: https://www.csie.ntu.edu.tw/~cjlin/libsvmtools/datasets/multilabel/wiki10_31k_tfidf_train.svm.bz2}
        \item Testing: \url{Testing: https://www.csie.ntu.edu.tw/~cjlin/libsvmtools/datasets/multilabel/wiki10_31k_tfidf_test.svm.bz2}
    \end{itemize}
    \item AmazonCat-13k:
    \begin{itemize}
        \item Training: \url{https://www.csie.ntu.edu.tw/~cjlin/libsvmtools/datasets/multilabel/AmazonCat-13K_tfidf_train_ver1.svm.bz2}
        \item Testing: \url{https://www.csie.ntu.edu.tw/~cjlin/libsvmtools/datasets/multilabel/AmazonCat-13K_tfidf_test_ver1.svm.bz2}
    \end{itemize}
    \item Amazon-670k:
    \begin{itemize}
        \item Training: \url{https://www.csie.ntu.edu.tw/~cjlin/libsvmtools/datasets/multilabel/Amazon-670K_tfidf_train_ver2.svm.bz2}
        \item Testing: \url{https://www.csie.ntu.edu.tw/~cjlin/libsvmtools/datasets/multilabel/Amazon-670K_tfidf_test_ver2.svm.bz2}
    \end{itemize}
    \item Wiki-500k and Amazon-3m are downloaded from the following GitHub repository provided in \citet{RY19a}
    \begin{itemize}
        \item \url{https://github.com/yourh/AttentionXML}
    \end{itemize}
\end{itemize}

\section{Empirical Observations on the Feature Reduction Ratio $\alpha$}
\label{sec:empirical-alpha}

Figure~\ref{fig:alpha-in-real-world} shows the histogram of the reduction ratio $\alpha_u$ for each internal node $u$, computed by
\begin{equation}
  \label{eq:alpha_mu_formulation}
  \dfrac{\text{\# used features of $u$}}{\text{\# used features of $u$'s parent}}.
\end{equation}
In the same figure we also show the weighted average of $\alpha$ for depth-$i$, denoted by $\bar{\alpha}$ and defined as 
\begin{equation}
  \label{eq:weighted-average-alpha}
  \bar{\alpha} = \dfrac{\sum\limits_{u\in \text{depth-}i} \alpha_u\cdot \text{\# children of $u$}}{\sum\limits_{u\in \text{depth-}i} \text{\# children of $u$}}.
\end{equation}
We use Figure~\ref{fig:weighted_avg} to illustrate the reason for reporting the weighted average.
When training nodes at depth-$1$, node B, C and D has respectively two, two and six weight vectors, so the average features used for each weight vector should be 
\begin{equation*}
  \label{eq:avg_feat_each_weight_vector}
  \bar{n} = \dfrac{10\cdot 2+30\cdot 2+80\cdot 6}{2+2+6},
\end{equation*}
and the average reduction ratio is 
\begin{align*}
  \frac{\bar{n}}{100}
  &= \frac{0.1 \cdot 2 + 0.3 \cdot 2 + 0.8 \cdot 6}{2 + 2 + 6}\\
  &= \frac{\alpha_B \cdot 2 + \alpha_C \cdot 2 + \alpha_D \cdot 6}{2 + 2 + 6},
\end{align*}
which is equivalent to \eqref{eq:weighted-average-alpha}.
\tikzset{
  hollow node/.style={circle,draw,inner sep=5},
}

For results in Figure~\ref{fig:alpha-in-real-world}, we construct the trees by setting $K=100$ and $d_{\text{max}}=6$.
The depth of the resulting tree for EUR-Lex and Amazon-13k is only 3 and 4 respectively because the construction has reached the termination condition.
Clearly, we see that most $\bar{\alpha}$'s are small, generally much lower than 0.5.
The only exception is in the depth-5 of the Amazon-3m set.
The reason of a relatively larger $\bar{\alpha}$ is that, under an unbalanced setting, some nodes at layer $d_{\text{max}}-1$ still contain many labels (much more than $K$) and need further partitioning, but the tree has reached $d_\text{max}$.
Therefore, these nodes have more used features than others in the same layer.
Then their $\alpha$ values from \eqref{eq:alpha_mu_formulation} are larger than others. 
Together with their larger number of children, these nodes dominate the calculation in \eqref{eq:weighted-average-alpha} and lead to a large weighted average.

\begin{figure*}[t]
    \centering
    \begin{subfigure}[h]{0.39\textwidth}
        \centering
        \includegraphics[width=\textwidth]{
          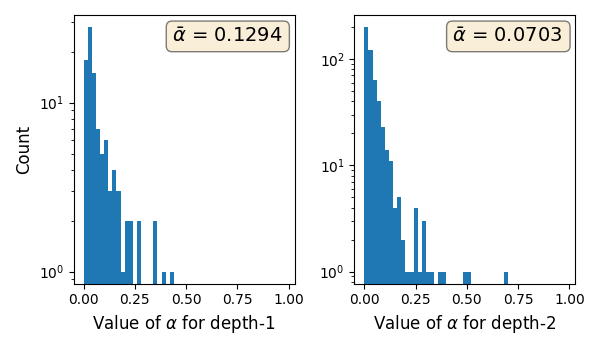}
        \vspace{-20pt}\caption{EUR-Lex}
    \end{subfigure}
    \begin{subfigure}[h]{0.59\textwidth}
        \centering
        \includegraphics[width=\textwidth]{
          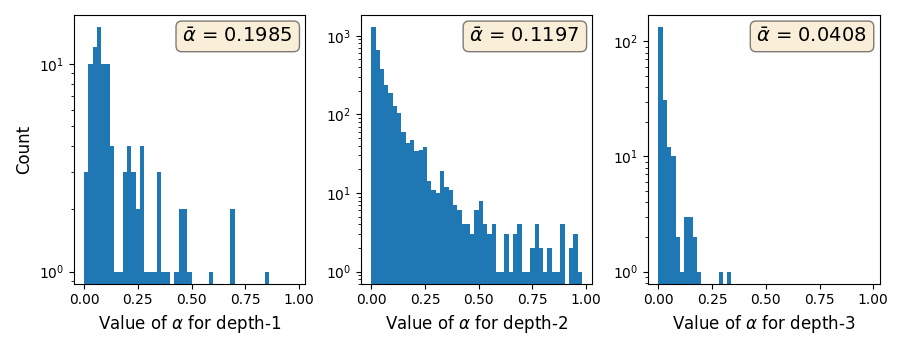}
        \vspace{-20pt}\caption{Amazoncat-13k}
    \end{subfigure} \hfill
    \vspace{0.3cm}
    \begin{subfigure}[h]{\textwidth}
      \centering
      \includegraphics[width=\textwidth]{
        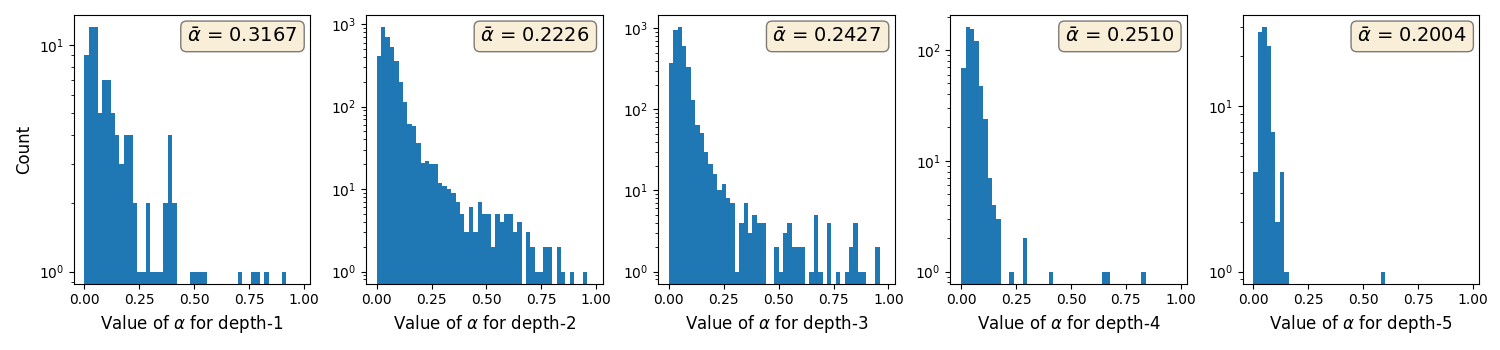}
        \vspace{-20pt}  \caption{Wiki10-31k}
    \end{subfigure} \hfill
    \vspace{0.3cm}
    \begin{subfigure}[h]{\textwidth}
        \centering
        \includegraphics[width=\textwidth]{
          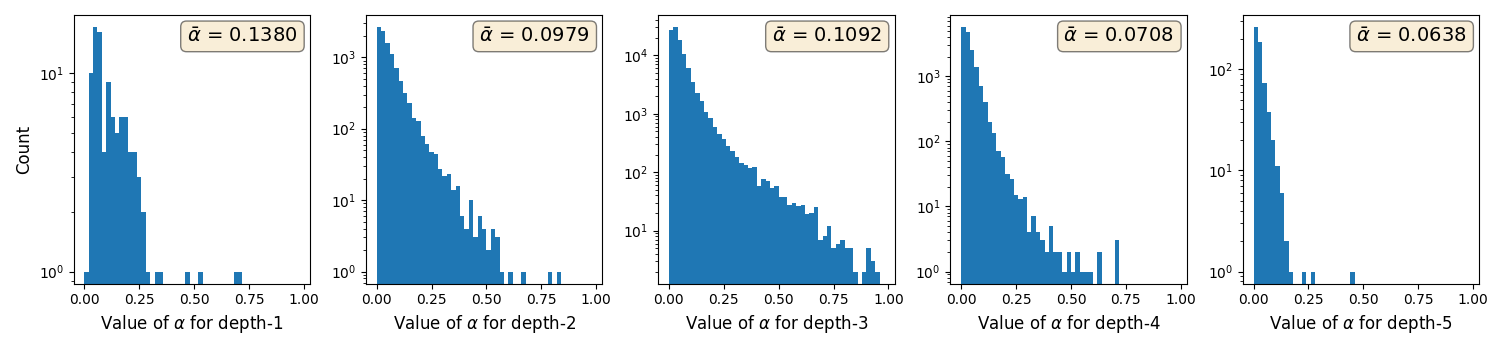}
          \vspace{-20pt}  \caption{Wiki-500k}
    \end{subfigure}  \hfill
    \vspace{0.3cm}
    \begin{subfigure}[h]{\textwidth}
        \centering
        \includegraphics[width=\textwidth]{
          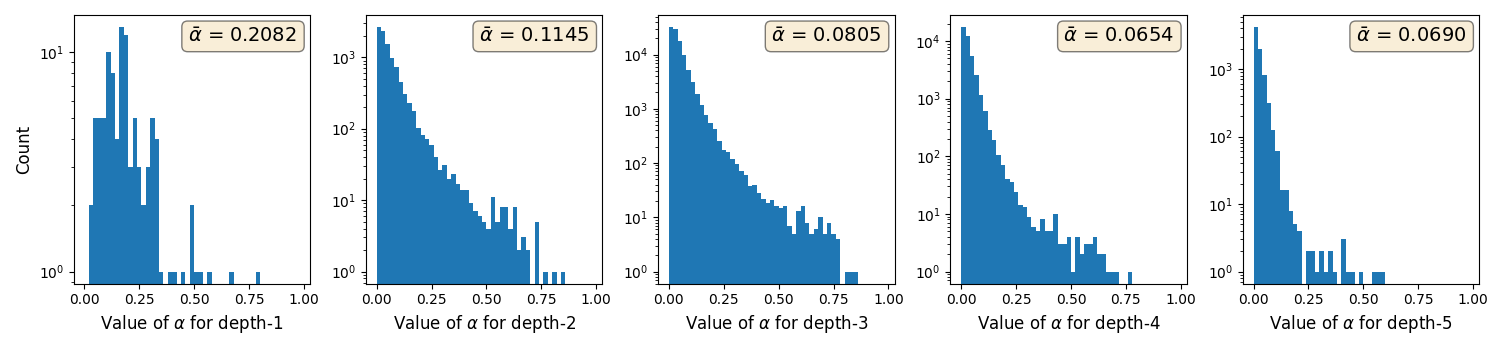}
          \vspace{-20pt}  \caption{Amazon-670k}
    \end{subfigure}  \hfill
    \vspace{0.3cm}
    \begin{subfigure}[h]{\textwidth}
        \centering
        \includegraphics[width=\textwidth]{
          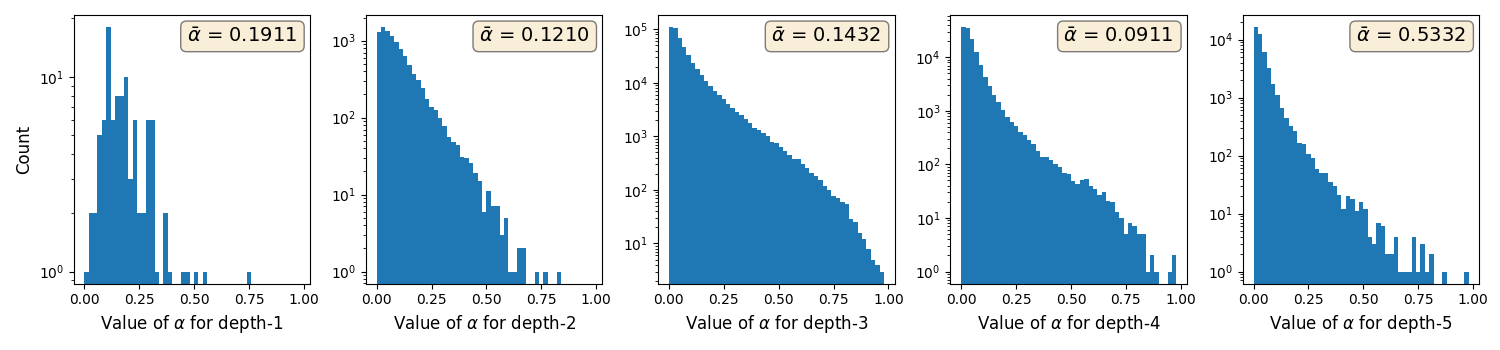}
          \vspace{-20pt}  \caption{Amazon-3m}
    \end{subfigure}  
    \caption{\label{fig:alpha-in-real-world}The histogram of $\alpha$ values of nodes with the same depth. The weighted average $\bar{\alpha}$ defined in \eqref{eq:weighted-average-alpha} is listed in the box of each subfigure.}
\end{figure*}

\begin{figure*}[t]
  \centering
  \begin{tikzpicture}
      [level distance=2cm,
          level 1/.style={sibling distance=3.8cm},
          level 2/.style={sibling distance=0.7cm}]
      \node [hollow node, label=above:{used features: {\color{red}100}}]{A}
      child {node [hollow node, label={[align=left,xshift=3.5em,yshift=0.3em]135:used features: {\color{red}10}}]{B}
          child {node {\shortstack{{\color{blue}1}\\\phantom{\{\}}}}
              child [grow=left, xshift=-1cm] {node {Depth 2} edge from parent[draw=none]
                  child [grow=up] {node {Depth 1} edge from parent[draw=none]
                      child [grow=up] {node {Depth 0} edge from parent[draw=none]}
                  }
              }
          }
          child {node {\shortstack{{\color{blue}2}\\\phantom{\{\}}}}}
      }
      child {node [hollow node, label=above:{used features: {\color{red}30}}]{C}
          child {node {\shortstack{{\color{blue}3}\\\phantom{\{\}}}}}
          child {node {\shortstack{{\color{blue}4}\\\phantom{\{\}}}}}
      }
      child {node [hollow node, label={[align=right,xshift=5.8em,yshift=0.3em]135:used features: {\color{red}80}}]{D}
          child {node {\shortstack{{\color{blue}5}\\\phantom{\{\}}}}}
          child {node {\shortstack{{\color{blue}6}\\\phantom{\{\}}}}}
          child {node {\shortstack{{\color{blue}7}\\\phantom{\{\}}}}}
          child {node {\shortstack{{\color{blue}8}\\\phantom{\{\}}}}}
          child {node {\shortstack{{\color{blue}9}\\\phantom{\{\}}}}}
          child {node {\shortstack{{\color{blue}10}\\\phantom{\{\}}}}}
      }
      ;
  \end{tikzpicture}
  \caption[A label tree with ten labels.]{
  A label tree with ten labels.
  At depth 1, nodes B, C, and D respectively train two, two, and six linear classifiers.
  We explain in Appendix~\ref{sec:empirical-alpha} that a weighted average of $\alpha$ values
  at nodes B, C, and D as in \eqref{eq:weighted-average-alpha} is a reasonable setting
  to calculate the reduction ratio for depth-1.
  }
  \label{fig:weighted_avg}
\end{figure*}
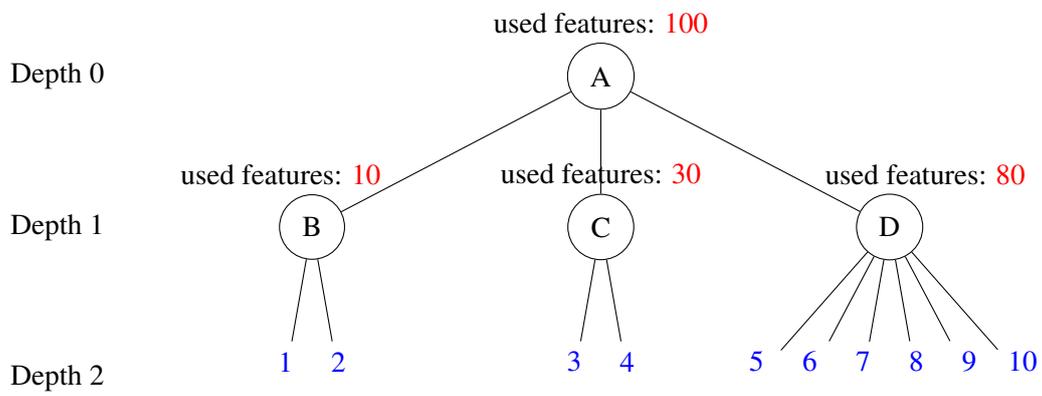


\end{document}